\documentclass{article}
\usepackage{arxiv_format,times}

\usepackage{amsmath,amsfonts,bm}

\def\eqref#1{equation~\ref{#1}}

\def\1{\bm{1}}

\DeclareMathAlphabet{\mathsfit}{\encodingdefault}{\sfdefault}{m}{sl}
\SetMathAlphabet{\mathsfit}{bold}{\encodingdefault}{\sfdefault}{bx}{n}

\newcommand{\KL}{D_{\mathrm{KL}}}

\usepackage{hyperref}
\usepackage{url}
\usepackage{cleveref}
\usepackage{amsmath,amssymb,amsthm,mathtools}
\usepackage{bm}
\usepackage{booktabs,multirow}
\usepackage{graphicx}
\usepackage{xcolor}
\usepackage{caption}
\usepackage{enumitem}
\usepackage{algorithm}
\usepackage{algpseudocode}
\usepackage{subcaption}
\usepackage{wrapfig}

\newtheorem{definition}{Definition}

\newtheorem{theorem}{Theorem}
\newtheorem{corollary}{Corollary}

\newtheorem{assumption}{Assumption}

\definecolor{rankblue}{HTML}{809EC2}
\definecolor{rankpink}{HTML}{F4C2C2}

\title{Overcoming Prior Barriers: Supervised Fine-Tuning under Long-Tail Distribution}

\author{
\makebox[\linewidth][c]{%
\textbf{Haohui Wang$^{1}$, Jiahao Xu$^{1}$, Wangzhi Zhan$^{1}$, Tong Zeng$^{2}$}}\\
\makebox[\linewidth][c]{%
\textbf{Dongqi Fu$^{3}$, Hong Li$^{3}$, Swastik Roy$^{2}$, Naren Ramakrishnan$^{1}$}}\\
\makebox[\linewidth][c]{%
\textbf{Chris North$^{1}$, Jian Kang$^{4}$, Yujun Yan$^{5}$, Dawei Zhou$^{1}$}}\\
\\[-0.2em]
\makebox[\linewidth][c]{%
$^{1}$Virginia Tech
\quad
$^{2}$Amazon
\quad
$^{3}$Meta
\quad
$^{4}$MBZUAI
\quad
$^{5}$Dartmouth College}
}

\begin{document}

\maketitle
\fancyhead{}
\renewcommand{\headrulewidth}{0pt}

\begin{abstract}
Supervised fine-tuning (SFT) adapts pretrained large language models (LLMs) to downstream tasks, but the required concepts can receive substantially different levels of pretrained support. Frequent concepts are more likely to be well learned, whereas rare concepts may remain weakly represented.
We introduce a novel notion named \emph{prior barrier} to quantify how strongly the pretrained model supports competing concepts over the target concept. 
We observe that prior barriers follow a long-tail distribution, placing head and tail concepts at different starting points for SFT: head concepts face lower prior barriers, whereas tail concepts require additional instructions to overcome their higher prior barriers.
Our theoretical analysis further derives a predictive risk bound for SFT under long-tail prior barriers, explicitly characterizing how the prior barrier and accumulated SFT evidence jointly determine predictive performance.
Motivated by this prior barrier-dependent demand, we propose \textsc{PASS}, an adaptive SFT instruction selection method that constructs reference-derived concepts and estimates the distinguishing evidence provided by each instruction, and adaptively allocates the selection budget toward concepts that remain insufficiently covered under the current selection. In this way, \textsc{PASS} jointly considers which instructions can provide useful evidence and where additional supervision is needed under a limited budget.
Experiments show that our method consistently outperforms seven state-of-the-art instruction selection methods on four backbone-budget settings. An ablation study further shows that \textsc{PASS}'s adaptive allocation consistently improves over uniform allocation.
\end{abstract}

\section{Introduction}
\label{sec:intro}
Supervised fine-tuning (SFT) is widely used to adapt pretrained large language models (LLMs) to downstream requirements such as mathematical reasoning, instruction following, factuality, and domain-specific behaviors~\citep{ouyang2022training,zhou2023lima, yue2023mammothbuildingmathgeneralist}. In practice, the available instruction pool can be much larger than the amount of data that can be used for fine-tuning, while candidate instructions can also vary substantially in quality, making data selection an important part of post-training~\citep{xia2024less}. Given a downstream task and a large pool of candidate instructions, which examples should be selected under a limited SFT budget? Recent studies show that carefully selected subsets can match or outperform substantially larger SFT datasets~\citep{zhou2023lima,chen2024alpagasus,xia2024less,liu2024makesgooddataalignment}, underscoring the importance of selecting valuable instructions under a limited SFT budget.

A central difficulty for effective SFT instruction selection is that the concepts required by the downstream task can receive substantially different levels of support from the pretrained model. 
Each task can involve a collection of concepts, ranging from common reasoning patterns and broadly observed instructions to rare knowledge and specialized behaviors. Frequently occurring concepts receive broader exposure during pretraining and are therefore more likely to receive strong support from the pretrained model, whereas rare concepts may remain weakly supported~\citep{pmlr-v202-kandpal23a,razeghi-etal-2022-impact,gekhman-etal-2024-fine}. We introduce a novel notion named \emph{prior barrier} to quantify how strongly the pretrained model supports competing concepts over the target concept. 
Different prior barriers place target concepts at different starting points for SFT: a larger prior barrier indicates that more instructions is needed to achieve comparable performance.
As illustrated in Figure~\ref{fig:intro-illustrative}, we observe a long-tail distribution in prior barriers, where head concepts face low prior barriers while tail concepts face high barriers, thus requiring more instructions.
Under a limited SFT budget, effective selection should determine how to allocate instructions across concepts and which candidate instructions provide useful evidence for each concept~\citep{chen2023skillitdatadrivenskillsframework, dai-etal-2025-improving}.

\begin{figure}[t]
    \centering
    \includegraphics[width=0.9\linewidth]{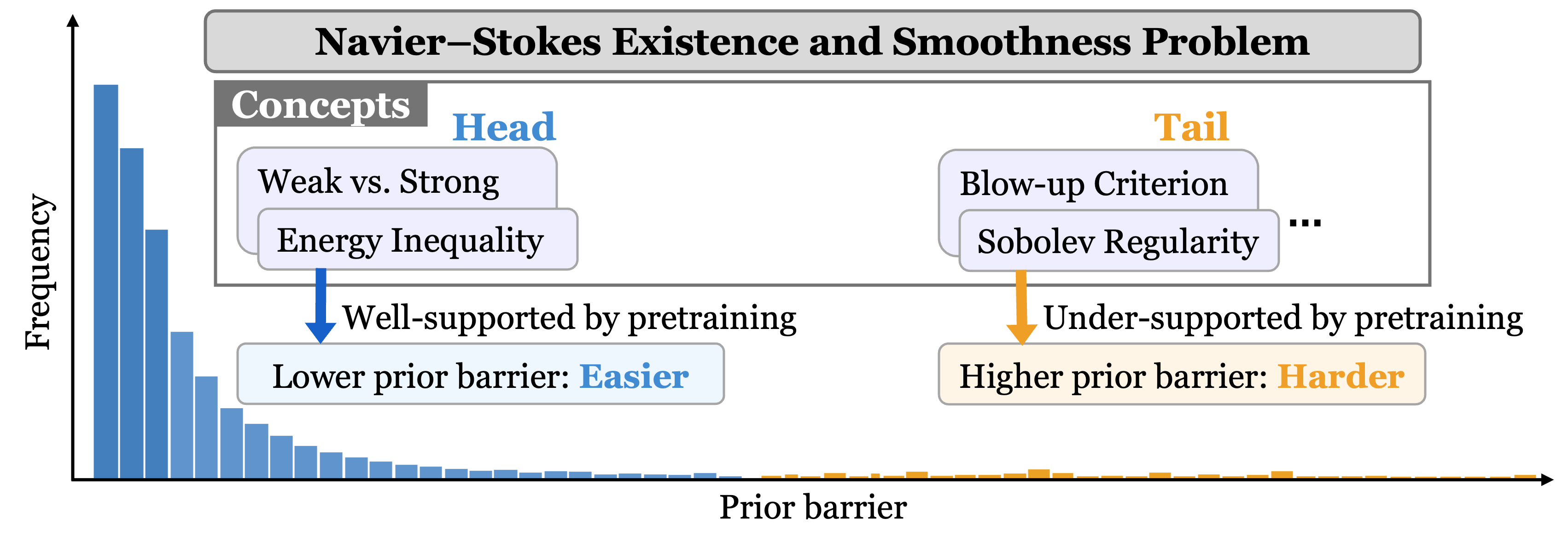}
    \caption{Illustration of the long-tail distribution of prior barriers using the Navier--Stokes existence and smoothness problem as an example.
    }
    \label{fig:intro-illustrative}
\end{figure}

These observations motivate two fundamental challenges: \textbf{C1. Theoretical foundation.} How does the SFT evidence overcome the prior barrier, and how does this determine the amount of instructions required to learn a target concept? \textbf{C2. Computational framework.} How can we develop a method that identifies which instructions in the candidate pool provide useful evidence for each concept while adaptively allocating a limited selection budget according to concept-level demand?

To address \textbf{C1}, we first characterize the role of the prior barrier in SFT. In the \emph{head case}, the target concept already receives sufficient pretrained support. SFT can move toward the target with relatively limited evidence. In the \emph{tail case}, the target concept faces a positive prior barrier, so the accumulated evidence must additionally overcome this disadvantage from pretraining. This interaction yields different instruction demands across concepts. We further derive a predictive risk bound that connects the demands to model performance by quantifying predictive risk in terms of the prior barrier and the distinguishing evidence provided by the selected instructions.

To address \textbf{C2}, we propose Prior-barrier Aware SFT Instruction Selection (\textsc{PASS}), an adaptive SFT instruction selection framework motivated by the prior barrier-dependent demand. It first constructs a reference-derived concept space and estimates concept-specific distinguishing evidence for each candidate instruction. \textsc{PASS} further estimates the prior barrier and existing pretrained support of each concept to characterize how strongly that concept calls for additional SFT instruction. It then adaptively allocates the selection budget toward concepts that remain under-represented. As a result, \textsc{PASS} jointly determines which candidate instructions provide useful evidence and where the limited SFT budget should be allocated.
Our empirical results show that \textsc{PASS} outperforms seven state-of-the-art instruction selection methods in all four backbone-budget settings. Ablation results further demonstrate the benefit of adaptively allocating the SFT budget according to prior barrier. The source code will be released upon publication.

\section{Preliminary}
\label{sec:preliminary}
In this section, we introduce the background, including a Bayesian concept-space view of SFT, and give the formal problem definition.
\paragraph{Notations.}
We consider a concept space $\Theta$, where each concept $\theta\in\Theta$ describes a possible input--response relation through a conditional distribution $P_\theta(y\mid x)$. Here, $x$ denotes an instruction input or query and $y$ denotes the corresponding response. For example, for a mathematical reasoning concept, $x$ may be a mathematical problem and $y$ its correct solution. We denote the target concept required by the downstream task as $\theta^\star$.
Pretraining does not provide equal support for all concepts. We represent this pretrained support by a distribution $\pi(\theta)$ over $\Theta$, where a larger $\pi(\theta)$ indicates stronger pretrained support for concept $\theta$. We interpret $\pi$ as a prior distribution over concepts induced by pretraining. As a result, different target concepts can begin SFT from substantially different pretrained states.

\paragraph{Supervised fine-tuning.}
SFT adapts a pretrained language model to the downstream task using a set of supervised instruction input--response pairs. 
Given an SFT dataset $\mathcal{D}_n=\{(x_i,y_i)\}_{i=1}^n$, where $x_i$ is the $i$-th instruction input and $y_i$ is its target response, SFT updates the pretrained model to increase the likelihood of the target responses conditioned on their inputs. For an autoregressive language model, this conditional probability factorizes as $P_\theta(y_i\mid x_i)=\prod_{t=1}^{T_i}
P_\theta(y_{i,t}\mid x_i,y_{i,<t})$, where $T_i$ is the number of tokens in response $y_i$, $y_{i,t}$ is the $t$-th token of $y_i$ and $y_{i,<t}$ denotes its preceding response tokens. Standard SFT therefore trains the model by minimizing the token-level negative log-likelihood over the SFT dataset. We summarize the evidence provided by the SFT dataset for concept $\theta$ through its log-likelihood $S_n(\theta)=\sum_{i=1}^n \log P_\theta(y_i\mid x_i)$.
We adopt a Bayesian concept-space view of SFT, where the pretrained distribution $\pi$ serves as the prior and the supervised data update this prior through their likelihood. Let $q$ denote a candidate distribution over the concept space $\Theta$ after observing the SFT dataset $\mathcal{D}_n$. We define $\mathcal{J}(q)
=
-\mathbb{E}_{\theta\sim q}\left[S_n(\theta)\right]
+
\KL(q\Vert\pi)$,
The posterior distribution is defined as
$q_n^\star
=
\arg\min_q \mathcal{J}(q)$,
which takes the Gibbs-posterior form
$q_n^\star(\theta)
=\frac{
e^{S_n(\theta)}\pi(\theta)
}{
\int_\Theta e^{S_n(u)}\pi(u)\,\lambda(du)
}$.
Here, $q_n^\star(\theta)$ represents the updated support for concept $\theta$, obtained by updating the pretrained prior $\pi(\theta)$ with the likelihood evidence from the SFT data through $S_n(\theta)$.

\paragraph{Problem definition.}
Given a candidate SFT pool $\mathcal{D}=\{(x_i,y_i)\}_{i=1}^{N}$, a selection budget $n$, and a reference set $\mathcal{T}$ representing the desired behaviors in the downstream task with a long-tail distribution of prior barriers across concepts, our goal is to find a subset of instructions $\mathcal{D}_n\subseteq\mathcal{D}$ with $
|\mathcal{D}_n|=n$ such that fine-tuning the pretrained model on $\mathcal{D}_n$ achieves strong overall performance on the downstream task.

\section{Theoretical Analysis}
\label{sec:theoretical-anlaysis}
In this section, we develop a theoretical analysis of how pretrained support and SFT evidence jointly shape predictive performance. First, we study posterior concentration around the target concept, showing how the prior barrier and distinguishing evidence together determine whether SFT can move the concept distribution toward the target. Second, we derive a predictive risk bound that connects posterior concentration to predictive performance and shows how the head--tail gap translates into different generalization guarantees.

\paragraph{Prior barrier.}
A pretrained model does not enter SFT from a neutral state: different concepts can receive substantially different levels of support from pretraining. This asymmetry is important because SFT adapts a pretrained model, where the target concept may compete with alternatives that are more strongly supported before fine-tuning begins. A key challenge is therefore to characterize this pretrained support relative to the target concept $\theta^\star$ and its competing concepts $\theta \in \Theta \setminus \{\theta^\star\}$.
To this end, we consider concepts outside an $r$-neighborhood of the target concept $A_r=\left\{
\theta\in\Theta:\|\theta-\theta^\star\|\ge r\right\}$. If $\theta^\star$ already receives stronger support than all concepts in $A_r$, the target concept does not face a pretrained disadvantage. Otherwise, some competing concept receives stronger pretrained support than $\theta^\star$, placing the target at an initial disadvantage before SFT begins. We characterize this relative pretrained disadvantage through the \emph{prior barrier}.
\begin{definition}[Prior barrier]
\label{def:prior-barrier}
The prior barrier at resolution $r$ is $B(r)=\sup_{\theta\in A_r}\log
\frac{\pi(\theta)}{\pi(\theta^\star)}$.
\end{definition}
The prior barrier $B(r)$ compares the pretrained support of the target concept with its strongest competing concept. We observe that prior barriers exhibit a long-tail distribution across concepts, with many concepts facing relatively small barriers and few concepts facing substantially larger ones. For the analysis below, we refer to $B(r)\le 0$ as the \emph{head case}, where the target concept is at least as well supported as all concepts in $A_r$, and $B(r)>0$ as the \emph{tail case}, where at least one competing concept receives stronger pretrained support than the target concept.

Overcoming a positive prior barrier requires the SFT instructions to contain information that distinguishes the target concept from its competing concepts. The amount of useful information depends on the instruction distribution $P_X$: some instructions can clearly expose the difference between the target concept and an alternative concept, whereas others provide little distinguishing evidence. For a competing concept $\theta$, define the expected evidence under $P_X$ as
\begin{equation}
K(\theta)
=
\mathbb{E}_{x\sim P_X}
\left[
\mathrm{KL}
\left(
P_{\theta^\star}(\cdot\mid x)
\Vert
P_{\theta}(\cdot\mid x)
\right)
\right]=
\mathbb{E}_{\substack{x\sim P_X\\ y\sim P_{\theta^\star}(\cdot\mid x)}}
\left[
\log P_{\theta^\star}(y\mid x)
-
\log P_{\theta}(y\mid x)
\right].
\label{eq:population-gap}
\end{equation}
Thus, a larger $K(\theta)$ means that the observed instruction input--response pair drawn from the instruction distribution provides stronger expected evidence for separating the target concept from $\theta$. Because SFT requires distinguishing the target from all concepts outside its neighborhood, we focus on the hardest competitor.

\begin{definition}[Distinguishing Evidence]
\label{def:identifiability}
The distinguishing evidence for the target concept at resolution $r$ is $\epsilon(r)=\inf_{\theta\in A_r} K(\theta)$.
\end{definition}
The target concept $\theta^\star$ is identifiable under the SFT instruction distribution $P_X$ when $\epsilon(r)>0$.

Learning the target concept requires the posterior distribution after SFT to place little mass on the competing concepts in $A_r$, that is, to make $q_n^\star(A_r)$ small. The two central quantities governing it are $B(r)$, which measures how much pretrained disadvantage must be overcome, and $n\epsilon(r)$, which measures how much distinguishing evidence SFT accumulates. As the number of SFT instructions increases, the accumulated evidence grows with $n\epsilon(r)$. Therefore, once this evidence is sufficient to overcome the pretrained disadvantage, the posterior can concentrate around the target concept.

\begin{theorem}[Posterior Concentration]
\label{thm:sft-contraction}
Fix $r>0$ such that $A_r\neq\varnothing$ and $\epsilon(r)>0$, and let $m(r)>0$ denote the normalized prior mass in a sufficiently small neighborhood of the target concept $\theta^\star$. Under the assumptions given in Appendix~\ref{app:assumptions}, for any $\eta\in(0,1)$, if $n\ge n_{\mathrm{sep}}(r,\eta)$ where $n_{\mathrm{sep}}(r,\eta)$ denotes the sample-size threshold, with probability at least $1-\eta$ over the SFT dataset $\mathcal{D}_n$,
\begin{equation}
q_n^\star(A_r)
\le
\min\left\{
1,
\frac{1}{m(r)}
\exp\left(
B(r)-\frac{n\epsilon(r)}{8}
\right)
\right\}.
\label{eq:sft-contraction}
\end{equation}
\end{theorem}
\begin{proof}
The proof is provided in Appendix~\ref{app:contraction-proof}.
\end{proof}

Theorem~\ref{thm:sft-contraction} shows that, once the target concept is identifiable under the SFT distribution, sufficiently many SFT examples drive posterior mass away from competing concepts and toward the target concept. Since the prior barrier is fixed while the accumulated distinguishing evidence grows with $n$, the concentration error decreases exponentially as supervision increases. The amount of supervision required, however, differs between the head and tail cases.

\begin{corollary}[Head--Tail Supervision Gap]
\label{cor:head-tail-sample}
Under the conditions of Theorem~\ref{thm:sft-contraction}, let $\delta\in(0,1)$ be a desired posterior error. The head case $B(r)\le0$ satisfies $q_n^\star(A_r)\le\delta$ whenever
\begin{equation}
n
\ge
\max\left\{
\frac{8}{\epsilon(r)}
\left(
\log\frac{1}{m(r)\delta}
\right)_+,
n_{\mathrm{sep}}(r,\eta)
\right\},
\label{eq:head-sample}
\end{equation}
where $(a)_+=\max\{a,0\}$. The tail case $B(r)>0$ satisfies $q_n^\star(A_r)\le\delta$ whenever
\begin{equation}
n
\ge
\max\left\{
\frac{8}{\epsilon(r)}
\left(
B(r)
+
\log\frac{1}{m(r)\delta}
\right)_+,
n_{\mathrm{sep}}(r,\eta)
\right\}.
\label{eq:tail-sample}
\end{equation}
\end{corollary}
\begin{proof}
The proof is provided in Appendix~\ref{app:head-tail-proof}.
\end{proof}
The posterior concentration result explains when SFT places most of its concept-level support near the target concept. We now ask whether this concept-level concentration also leads to predictive performance. Define the posterior predictive distribution as $P_{q_n^\star}(y\mid x)=\int_{\Theta}P_{\theta}(y\mid x)q_n^\star(d\theta)$, and measure its excess predictive log-loss relative to the target concept by $\mathcal{E}_n =\mathbb{E}_{x\sim P_X}\left[\KL\left(P_{\theta^\star}(\cdot\mid x)\Vert P_{q_n^\star}(\cdot\mid x)\right)\right]$. We now establish how posterior concentration around the target concept translates into predictive performance for SFT. 
\begin{theorem}[Predictive Risk Bound]
\label{thm:risk-decomposition}
Let $\alpha(r)=\mathbb{E}_{x\sim P_X}\left[\sup_{\|\theta-\theta^\star\|<r} \KL\left(P_{\theta^\star}(\cdot\mid x)\Vert P_{\theta}(\cdot\mid x)\right)\right]$ denote the worst local predictive discrepancy within distance $r$ of the target concept, and let $M_r<\infty$ be a uniform upper bound on the predictive KL discrepancy for competing concepts in $A_r$. Under the assumptions stated in Appendix~\ref{app:assumptions}, SFT dataset $\mathcal{D}_n$ satisfies
\begin{equation}
\mathcal{E}_n
\le
\alpha(r)
+
M_r q_n^\star(A_r).
\label{eq:risk-decomposition}
\end{equation}
Consequently, under the conditions of Theorem~\ref{thm:sft-contraction}, with probability at least $1-\eta$,
\begin{equation}
\mathcal{E}_n
\le
\alpha(r)
+
M_r
\min\left\{
\frac{1}{m(r)}
\exp\left(
B(r)-\frac{n\epsilon(r)}{8}
\right),
1
\right\}.
\label{eq:generalization-bound}
\end{equation}
\end{theorem}
\begin{proof}
The proof is provided in Appendix~\ref{app:risk-proof}.
\end{proof}
Theorem~\ref{thm:risk-decomposition} shows that predictive risk depends on two terms: how well concepts near the target concept approximate its predictions, and how much posterior mass remains on competing concepts. As SFT concentrates the posterior around the target concept, the contribution from competing concepts decreases, leading to a tighter predictive risk bound. In particular, if SFT achieves $q_n^\star(A_r)\le\delta$, then $\mathcal{E}_n
\le \alpha(r)+M_r\delta$. Thus, a target concept in the head case can reach a given predictive risk with less supervision, while a target concept in the tail case must first accumulate additional evidence to overcome its positive prior barrier. As a result, the tail case generally requires more SFT instructions to reach the same level of predictive risk.

\begin{figure}[t]
    \centering
    \includegraphics[width=0.9\linewidth]{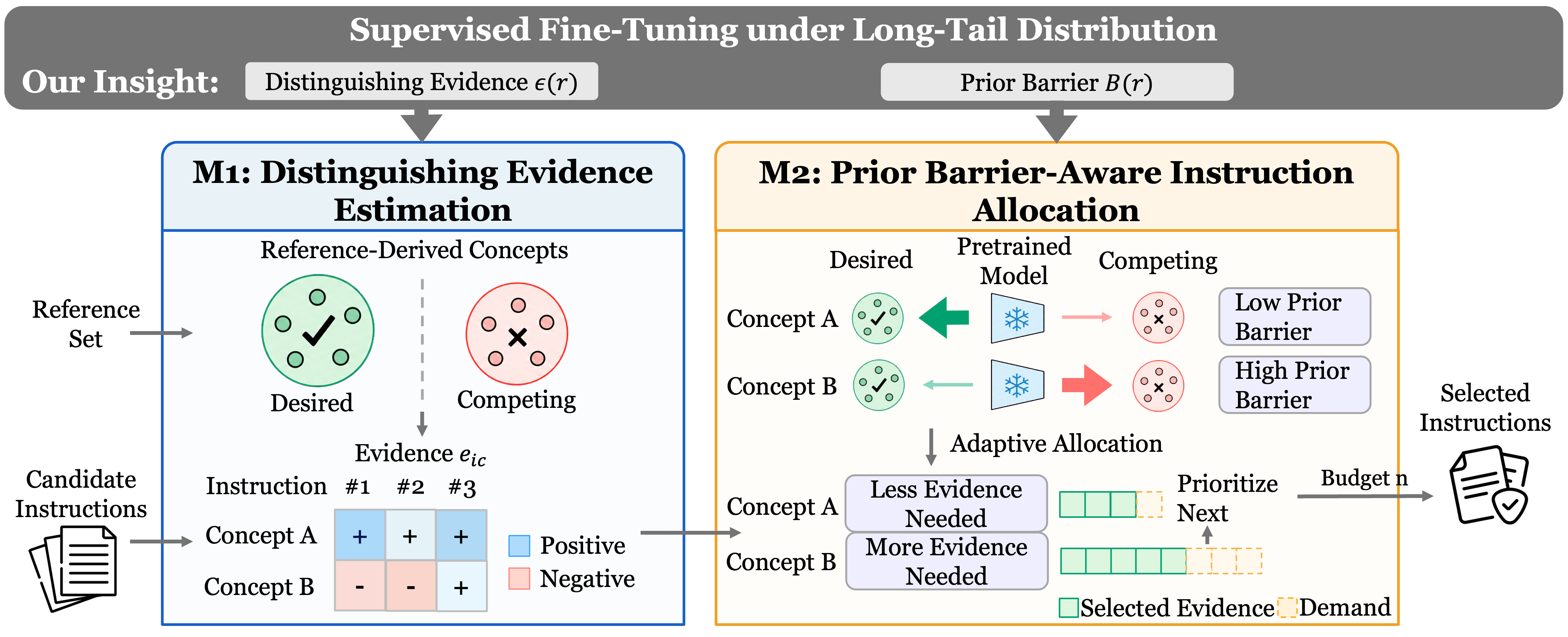}
    \caption{The proposed prior-barrier aware SFT instruction selection (\textsc{PASS}) framework.}
    \label{fig:method-overview}
\end{figure}

\section{Method}
\label{sec:method}
Our theoretical analysis shows that achieving a small predictive risk requires SFT instructions to provide sufficient evidence for distinguishing the target concept from competing concepts. It further reveals different instruction requirements in the head and tail cases: a tail target concept requires additional instructions to overcome its positive prior barrier. These results motivate us to consider both how much distinguishing evidence each candidate instruction provides and how a fixed instruction budget should be allocated across target concepts with asymmetric pretrained support. In this section, we introduce \textsc{PASS} (shown in Figure~\ref{fig:method-overview}), which consists of two major modules: \emph{M1. Distinguishing evidence estimation} and \emph{M2. Prior barrier-aware instruction allocation}. Accordingly, M1 constructs reference-derived concepts and estimates evidence for each candidate instruction by contrasting its support for the desired concept against relevant competing concepts. M2 uses the pretrained model to estimate the barrier and existing support of each concept, and incrementally allocates instructions according to both concept-level demand.

\subsection{M1: Distinguishing Evidence Estimation}
Given a reference set $\mathcal{T}$, we first construct a finite empirical discretization of the concept space $\Theta$ introduced in Section~\ref{sec:preliminary}.
We encode each reference instruction input $x$ together with its target response $y$ using the pretrained model $M_0$ to extract a response representation $z(x,y)\in\mathbb{R}^{d}$. We hierarchically cluster these reference representations, recursively splitting clusters that do not satisfy a coherence criterion. We measure coherence by the concentration of cosine similarity between the representations in a cluster and their centroid, and stop splitting once each cluster is sufficiently coherent. The resulting clusters instantiate a finite set of target concepts $\Theta={\theta_1,\ldots,\theta_C}$, where $C$ is the number of concepts. Each concept $\theta$ therefore corresponds to a subset $\mathcal{T}_\theta\subseteq\mathcal{T}$ of reference instructions that exhibit a coherent representation.

We represent the desired representation of concept $\theta$ by the normalized centroid of the representations in $\mathcal{T}_\theta$:
\begin{equation}
    p_\theta
    =
    \operatorname{Norm}
    \left(
        \frac{1}{|\mathcal{T}_\theta|}
        \sum_{(x,y)\in\mathcal{T}_\theta} z(x,y)
    \right),
\end{equation}
where $\operatorname{Norm}(v)=\frac{v}{\|v\|_2}$. 
To construct competing representations, we keep the same instruction inputs in $\mathcal{T}_\theta$ but replace their target responses with corresponding wrong responses, using available incorrect alternatives together with incorrect responses produced by $M_0$. These wrong responses may reflect multiple competing concepts for the same concept. For concept $\theta$, let $\mathcal{W}_{\theta k}$ denote the set of wrong response instructions corresponding to its $k$-th competing concept, with $k=1,\ldots,K_\theta$, where $K_\theta$ denotes the number of competing concepts associated with target concept $\theta$. We represent each competing representation by the normalized centroid:
\begin{equation}
    n_{\theta k}
    =
    \operatorname{Norm}
    \left(
        \frac{1}{|\mathcal{W}_{\theta k}|}
        \sum_{(x,\tilde y)\in\mathcal{W}_{\theta k}} z(x,\tilde y)
    \right),
    \qquad
    k=1,\ldots,K_\theta,
    \label{eq:pass_competing_representation}
\end{equation}
where $\tilde y$ denotes a wrong response. The resulting set $\mathcal{N}_\theta=\{n_{\theta 1},\ldots,n_{\theta K_\theta}\}$ defines the competing representations of concept $\theta$.

Our theoretical analysis suggests that a instruction may provide different distinguishing evidence for different concepts. We therefore evaluate candidate instruction $(x_i,y_i)\in\mathcal{D}$ separately against every concept. Let $z_i=z(x_i,y_i)$ be the representation of the $i$-th instruction. For concept $\theta$, we measure its similarity to the desired representation and to the closest competing representation as
\begin{equation}
    s_\theta^{+}(z_i)=\operatorname{Sim}(z_i,p_\theta),
    \qquad
    s_\theta^{-}(z_i)=\max_{v\in\mathcal{N}_\theta}\operatorname{Sim}(z_i,v),
    \label{eq:pass_similarity}
\end{equation}
where $\operatorname{Sim}$ denotes cosine similarity. The difference $s_\theta^{+}(z_i)-s_\theta^{-}(z_i)$ indicates whether the candidate favors the desired representation of concept $\theta$ over its strongest competitor.
For directly comparable across concepts, we calibrate this difference using statistics estimated from the reference responses of each concept. We define the empirical distinguishing evidence as
\begin{equation}
    \epsilon_{i\theta}
    =
    \operatorname{Calib}_\theta
    \left(
        s_\theta^{+}(z_i),
        s_\theta^{-}(z_i)
    \right),
    \label{eq:pass_evidence}
\end{equation}
with $\operatorname{Calib}_\theta(s^{+},s^{-})=    r_\theta(s^{+},s^{-})\frac{s^{+}-s^{-}}{\sigma_\theta}$. Here, $r_\theta(s^{+},s^{-})\in[0,1]$ is a relevance factor determined by the candidate's affinity $\max\{s^{+},s^{-}\}$ relative to the affinity observed for the reference responses of concept $\theta$. It downweights candidates that are far from both the desired and competing representations while preserving candidates that are strongly aligned with either side. The scale $\sigma_\theta>0$ normalizes the similarity difference between $s_\theta^{+}$ and $s_\theta^{-}$ for concept $\theta$, making evidence magnitudes comparable across concepts. Directly estimating the theoretical distinguishing evidence requires conditional response distributions, which are generally unavailable in practice. We therefore use the calibrated similarity margin as an empirical surrogate. A positive $\epsilon_{i\theta}$ indicates that candidate $i$ provides evidence in favor of the desired representation of concept $\theta$, whereas a negative value indicates stronger support for a competing concept. 
M1 outputs the evidence matrix $[\epsilon_{i\theta}]\in\mathbb{R}^{N\times C}$ for the reference-derived concepts rather than collapsing each candidate into a single score.

\subsection{M2: Prior Barrier-Aware Instruction Allocation}
M1 identifies how much distinguishing evidence each candidate can provide for different concepts, while the remaining question is how to allocate a fixed selection budget across concepts with different demands. Our theoretical analysis shows that a tail concept requires additional instructions to overcome its pretrained disadvantage. M2 therefore estimates concept-level demand from the prior barrier and incrementally allocates instructions as useful evidence is accumulated.

Since the theoretical prior barrier is not directly observable, we estimate it empirically from the pretrained model's responses. We let the pretrained model $M_0$ generate a response $\hat y\sim M_0(\cdot\mid x)$ and obtain its representation $\hat z=z(x,\hat y)$ for each instruction input in $\mathcal{T}_\theta$. 
The prior barrier of concept $\theta$ is estimated from the normalized difference between competing and desired representations,
\begin{equation}
    B_\theta
    =
    \operatorname{Agg}_{(x,y)\in\mathcal{T}_\theta}
    \left[
        \frac{s_\theta^{-}(\hat z)-s_\theta^{+}(\hat z)}{\sigma_\theta}
    \right],
    \label{eq:pass_barrier}
\end{equation}
where $\operatorname{Agg}$ denotes the aggregation over the reference instructions of concept $\theta$. 

Complementary to this relative barrier, we separately estimate the absolute pretrained support for the target concept by measuring whether the generated response $\hat y$ lies sufficiently close to the desired representation $p_\theta$:
\begin{equation}
\gamma_\theta
=
\frac{1}{|\mathcal{T}_\theta|}
\sum_{(x,y)\in\mathcal{T}_\theta}
\operatorname{Support}_\theta
\left(
\operatorname{Sim}(\hat z,p_\theta)
\right),
\label{eq:pass_support}
\end{equation}
where $\operatorname{Support}_\theta(\cdot)\in[0,1]$ calibrates the similarity to $p_\theta$ using the similarities observed from the target responses of concept $\theta$. A larger $\gamma_\theta$ indicates stronger pretrained support for the target concept.

A concept should receive more instruction when the pretrained model exhibits a higher prior barrier and weaker absolute target support. We therefore define its allocation weight as
\begin{equation}
    A_\theta
    =
    \frac{\exp(B_\theta)/\gamma_\theta}
    {\sum_{\theta'\in\Theta}\exp(B_{\theta'})/\gamma_{\theta'}},
    \label{eq:pass_allocation_weight}
\end{equation}
so that $\sum_{\theta\in\Theta}A_\theta=1$. While $A_\theta$ characterizes the relative instruction demand of each concept, the final selected subset should also account for how selected instructions contribute positive and negative distinguishing evidence, as well as the redundancy among their representations. We therefore define the selection objective by jointly modeling these three factors:
\begin{equation}
    F(\mathcal{S})
    =
    \sum_{\theta\in\Theta} A_\theta R_\theta^{+}(\mathcal{S})
    -
    \lambda_h
    \sum_{\theta\in\Theta} A_\theta R_\theta^{-}(\mathcal{S})
    -
    \lambda_r D(\mathcal{S}),
    \label{eq:pass_objective}
\end{equation}
where $R_\theta^{+}$ rewards positive distinguishing evidence for concept $\theta$, $R_\theta^{-}$ penalizes negative distinguishing evidence for concept $\theta$, $D$ penalizes redundant selections, and $\lambda_h$ and $\lambda_r$ denote the weights of harm penalty and redundancy penalty.

For positive evidence, repeatedly selecting instructions for the same concept should yield diminishing returns. We therefore define $R_\theta^{+}(\mathcal{S})=1-\exp\left\{-\beta^{+}_\theta\sum_{i\in\mathcal{S}} \epsilon_{i\theta}^{+}\right\}$, where $\epsilon_{i\theta}^{+}=(\epsilon_{i\theta})_{+}$. The coefficient $\beta^{+}_\theta>0$ controls how quickly the positive reward for concept $\theta$ saturates. Concepts with richer available evidence therefore saturate more slowly, while the marginal reward decreases as more positive evidence has already been selected.

For negative evidence, we define $R_\theta^{-}(\mathcal{S})=\sum_{i\in\mathcal{S}}\left[    \exp\left\{\beta^{-} \epsilon_{i\theta}^{-}\right\}-1\right]$, where $\epsilon_{i\theta}^{-}=(-\epsilon_{i\theta})_{+}$, and $\beta^{-}>0$ is an evidence scale shared across concepts for negative contributions. This term assigns a larger penalty to instructions that more strongly favor competing concepts.

Finally, we discourage selecting instructions that provide highly redundant information. We define $D(\mathcal{S})=\frac{2}{|\mathcal{S}|(|\mathcal{S}|-1)}
\sum_{\substack{i,j\in\mathcal{S}}}
\operatorname{Sim}(z_i,z_j),$. This term encourages the selected subset to cover distinct regions of the response representation space.

We initialize $\mathcal{S}=\varnothing$ and greedily select the instruction with the largest marginal gain,
\begin{equation}
    i^{\star}
    \in
    \arg\max_{i:(x_i,y_i)\in\mathcal{D}\setminus\mathcal{S}}
    \left[
        F\bigl(\mathcal{S}\cup\{(x_i,y_i)\}\bigr)-F(\mathcal{S})
    \right].
    \label{eq:pass_greedy_selection}
\end{equation}
After each selection, the selection objective is updated. The procedure stops when $|\mathcal{S}|=n$, yielding the subset used for SFT. Algorithm~\ref{alg:pass} in Appendix~\ref{app:pseudocode} summarizes the complete procedure.

\section{Experiments}
\label{sec:experiments}
In this section, we evaluate the performance of our method on four backbone-budget settings. \textsc{PASS} exhibits superior performance compared to various state-of-the-art baselines (Section~\ref{sec:results}). Further, through the ablation study (Section~\ref{sec:ablation}), we demonstrate the necessity of each component of \textsc{PASS}. We also report parameter analysis (Section~\ref{sec:parameter}), complexity sensitivity (Appendix~\ref{sec:complexity}), and the composition of instructions selected by different instruction selection methods (Appendix~\ref{app:composition}).

\subsection{Experiment Setup}
\label{sec:setup-exp}

\textbf{Candidate pool and target tasks.}
The candidate pool contains 300,932 instruction–response examples drawn from five open instruction datasets: FLAN V2~\citep{pmlr-v202-longpre23a}, OASST1~\citep{NEURIPS2023_949f0f8f}, WizardLM~\citep{ICLR2024_82eec786}, Dolly~\citep{DatabricksBlog2023DollyV2} and Stanford Alpaca~\citep{alpaca}. 
To evaluate SFT instruction selection across concepts with long-tail prior barrier, we construct a downstream task spanning five capabilities: broad knowledge, multi-step reasoning, mathematics, truthfulness, and multilingual extractive QA. We form the reference set from examples covering these capabilities so that instruction selection is evaluated against the full range of target concepts.
As shown in Figure~\ref{fig:barrier-distribution}, the prior barrier scores of the concepts in the reference set exhibit a long-tail distribution, providing a suitable setting for evaluating SFT under the long-tail distribution.

\begin{figure}[t]
\centering
\includegraphics[width=0.45\textwidth]{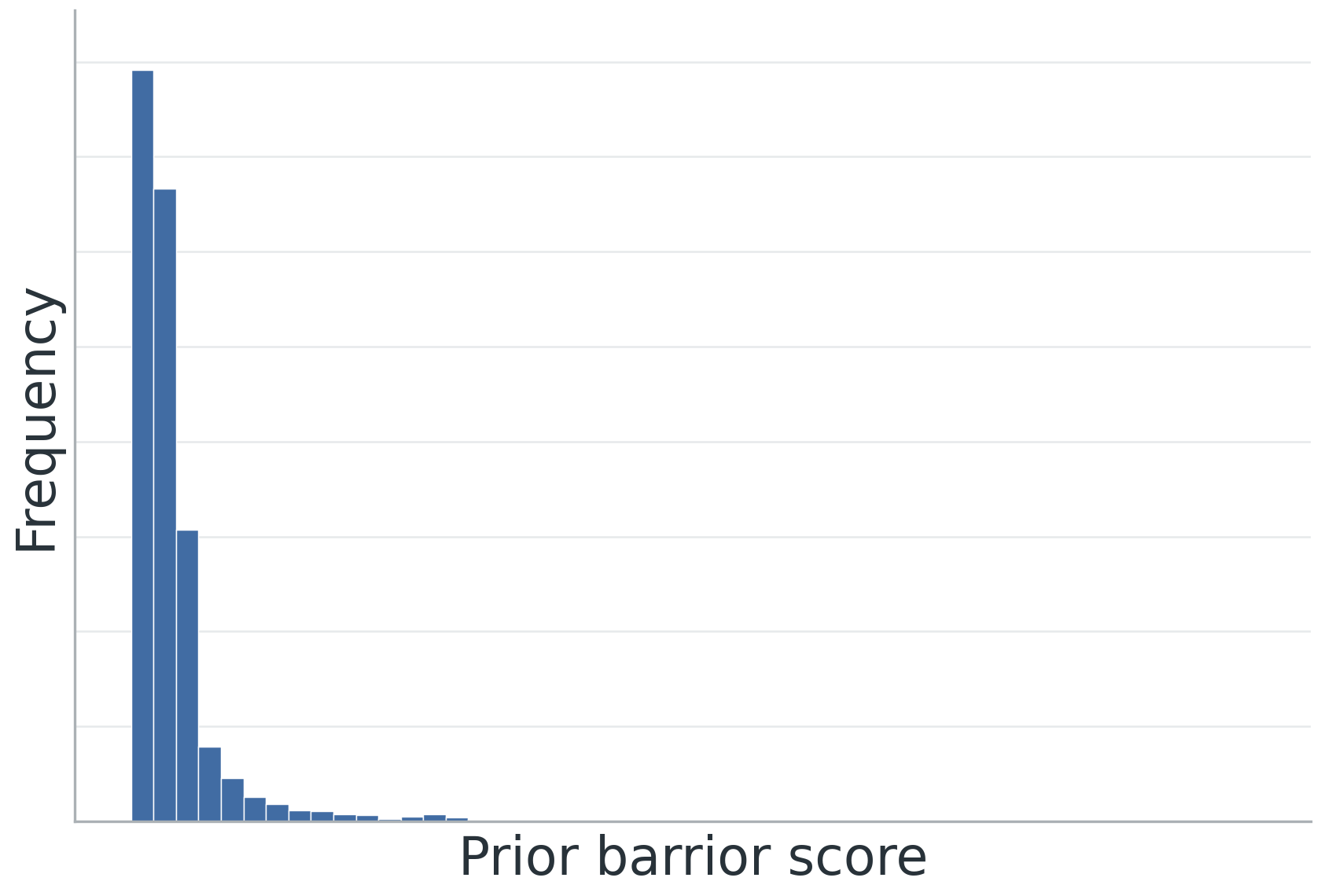}
\caption{Long-tail distribution of prior barrier over the reference-derived concepts.}
\label{fig:barrier-distribution}
\end{figure}

\textbf{Backbones, budgets, and metrics.}
The main comparisons use Mistral-7B~\citep{jiang2023mistral} and Qwen2.5-7B~\citep{qwen2024qwen25}. We select $n\in\{5{,}000,10{,}000\}$ records, corresponding to approximately $1.66\%$ and $3.32\%$ of the original candidate pool. Within each backbone and budget, all selectors share the fine-tuning settings, the training recipe uses three epochs of LoRA~\citep{hu2022lora}. We report the overall performance on the downstream task (the macro-average of the five capability benchmark scores).

\textbf{Baselines.}
We compare against \emph{Base} (backbone without fine-tuning) and seven instruction selection methods, including \textsc{Random} that samples uniformly; representation-based \textsc{RDS+}~\citep{ivison2025large}; gradient-based \textsc{LESS}~\citep{xia2024less}; \textsc{TSDS} that uses optimal transport~\citep{liu2024tsds}; \textsc{TADS} that combines divisity selections by voting~\citep{cheng2026tads}; answer-divergence-based \textsc{ADG}~\citep{li2026instruction}; and instance-quality filtering by \textsc{DEITA}~\citep{liu2024makesgooddataalignment}.

\begin{table}[t]
  \centering
  \begin{minipage}{\linewidth}
  \centering
  \caption{Overall downstream performance on Mistral-7B and Qwen2.5-7B at two selection budgets. The best and second-best results within each model and budget are shown in \textbf{bold} and \underline{underlined}.}
  \label{tab:main-results}
  \setlength{\tabcolsep}{7pt}
  \renewcommand{\arraystretch}{1.08}
  \begin{tabular*}{\linewidth}{@{\extracolsep{\fill}}lrrrr@{}}
    \toprule
    & \multicolumn{2}{c}{Mistral-7B} & \multicolumn{2}{c}{Qwen2.5-7B} \\
    \cmidrule(lr){2-3}\cmidrule(l){4-5}
    Method & $n=5{,}000$ & $n=10{,}000$ & $n=5{,}000$ & $n=10{,}000$ \\
    \midrule
    \textsc{Base} (no SFT) & 26.9 & 26.9 & 52.8 & 52.8 \\
    \midrule
    \textsc{Random}        & 49.3 & 50.4 & 61.7 & \textbf{65.6} \\
    \textsc{RDS+}          & 49.9 & 50.3 & 64.3 & 63.4 \\
    \textsc{LESS}         & 47.7 & 50.3 & 60.3 & 59.9 \\
    \textsc{TSDS}          & 49.4 & 49.8 & 63.1 & 63.0 \\
    \textsc{TADS} & 48.7 & \underline{50.9} & \underline{64.6} & 63.2 \\
    \textsc{ADG}           & 50.0 & \underline{50.9} & 62.4 & 64.1 \\
    \textsc{DEITA} & \underline{50.6} & 48.9 & 60.7 & 62.8 \\
    \midrule
    \textsc{PASS} & \textbf{51.5} & \textbf{51.8} & \textbf{64.8} & \textbf{65.6} \\
    \bottomrule
  \end{tabular*}
  \end{minipage}

  \begin{minipage}{\linewidth}
  \centering
  \caption{Budget-aware instruction allocation versus uniform instruction allocation. $\Delta_{\mathrm{uni}}=\textsc{PASS}-\textsc{Uniform}$. Bold identifies the better results.}
  \label{tab:uniform-allocation}
  \small
  \setlength{\tabcolsep}{4pt}
  \renewcommand{\arraystretch}{1.08}
  \begin{tabular*}{\linewidth}{@{\extracolsep{\fill}}lrrrrrr@{}}
    \toprule
    & \multicolumn{3}{c}{$n=5{,}000$} & \multicolumn{3}{c}{$n=10{,}000$} \\
    \cmidrule(lr){2-4}\cmidrule(l){5-7}
    Backbone & \textsc{PASS} & \textsc{Uniform} & $\Delta_{\mathrm{uni}}$ & \textsc{PASS} & \textsc{Uniform} & $\Delta_{\mathrm{uni}}$ \\
    \midrule
    Mistral-7B & \textbf{51.46} & 51.14 & $+0.32$ & \textbf{51.81} & 51.54 & $+0.26$ \\
    Qwen2.5-7B     & \textbf{64.77} & 64.48 & $+0.29$ & \textbf{65.60} & 65.03 & $+0.56$ \\
    \bottomrule
  \end{tabular*}
  \end{minipage}

  \begin{minipage}{\linewidth}
    \centering
    \caption{Parameter sensitivity of \textsc{PASS}. $\Delta_{\mathrm{def}}$ is the difference relative to the default setting.}
    \label{tab:pass-parameter-sensitivity}
    \small
    \setlength{\tabcolsep}{7pt}
    \renewcommand{\arraystretch}{1.08}
    \begin{tabular}{lcccc}
        \toprule
        Setting & $\lambda_h$ & $\lambda_r$ & Mean & $\Delta_{\mathrm{def}}$ \\
        \midrule
        Default               & 1.0 & 0.10 & 51.46 & $\phantom{+}0.00$ \\
        Weak harm penalty  & 0.5 & 0.10 & 51.56 & $+0.10$ \\
        Strong harm penalty  & 2.0 & 0.10 & 50.54 & $-0.92$ \\
        Weak redundancy penalty    & 1.0 & 0.05 & 50.96 & $-0.50$ \\
        Strong redundancy penalty   & 1.0 & 0.20 & 52.21 & $+0.75$ \\
        \bottomrule
    \end{tabular}
  \end{minipage}
\end{table}

\subsection{Performance Analysis}
\label{sec:results}
Table~\ref{tab:main-results} shows that \textsc{PASS} consistently achieves the best overall performance across all four settings, indicating consistent gains across backbone choice and selection budget. In particular, on Mistral-7B, \textsc{PASS} reaches a score of $51.5$ with $5{,}000$ selected instructions,
higher than every competing selector using twice the budget ($n=10{,}000$), where the strongest overall performance is $50.9$. This result indicates that a better composition of the selected subset can compensate for a substantially smaller SFT budget. In contrast, other selection methods may identify individually relevant or high-quality instructions, but they fail to adaptively cover all capabilities required for downstream tasks, resulting in poor overall performance. This is particularly evident on Qwen2.5-7B with $n=10{,}000$, where \textsc{Random} achieves an average score of $65.6$, outperforming all other baselines. Taken together, these results support the motivation behind \textsc{PASS}: the key challenge is not only to identify useful individual instructions, but also to adaptively allocate budget based on the prior barrier. 

\subsection{Ablation Study}
\label{sec:ablation}
We compare our budget-aware instruction allocation strategy with \emph{uniform instruction allocation}. Under uniform variant, we allocate an equal budget of $n/5$ separately to each of the five capabilities in the downstream task. Table~\ref{tab:uniform-allocation} shows that budget-aware allocation consistently outperforms uniform allocation across both backbones and both selection budgets. On Mistral-7B, \textsc{PASS} improves the average score by $0.32$ and $0.26$ points at $n=5{,}000$ and $n=10{,}000$, respectively. A similar trend is observed on Qwen2.5-7B, where the gains are $0.29$ and $0.56$ points. These results indicate that an equal per-capability assignment is suboptimal. In contrast, \textsc{PASS} adaptively allocates instructions according to the prior barrier, thereby making better use of the limited SFT budget.

\subsection{Parameter Analysis}
\label{sec:parameter}
We study two hyperparameters of
our method: (1) the harm coefficient $\lambda_h$ weights the penalty on negative distinguishing evidence, and (2) the redundancy coefficient $\lambda_r$ controls the penalty on redundant selections. Table~\ref{tab:pass-parameter-sensitivity} reports the overall performance under varying $\lambda_h$ and $\lambda_r$ using Mistral-7B with $n=5{,}000$ selected instructions. The fluctuation of the performance is less than $2\%$, showing \textsc{PASS} is reliable and not sensitive to the hyperparameters over a wide range.

\section{Related Work}
\label{sec:related}

Instruction selection for supervised fine-tuning (SFT) has been studied from several perspectives, including instance quality, model-dependent utility, target-aware matching. Early work shows that small but carefully selected instruction sets can be highly effective~\citep{zhou2023lima,chen2024alpagasus}. Methods such as DEITA further combine response quality, instruction complexity, and embedding diversity to identify informative and non-redundant examples~\citep{liu2024makesgooddataalignment}. Other approaches make selection dependent on the model itself. For example, ADG measures the divergence among multiple generated answers to estimate instruction utility~\citep{li2026instruction}, while related methods use model difficulty, uncertainty, or response likelihood as selection signals~\citep{li2024quantity,liu2024selectit,zhang2025grape}.

Recent work uses downstream references to identify instructions that are relevant to the target task. LESS estimates candidate influence through low-rank gradient similarity to reference examples~\citep{xia2024less}; RDS+ retrieves candidates using pretrained representations~\citep{ivison2025large}; TSDS aligns the selected subset with the target distribution through optimal transport~\citep{liu2024tsds}; and TADS uses LLM-inferred proxy labels to improve task-aware distribution matching~\citep{cheng2026tads}. These methods provide strong mechanisms for estimating whether an instruction is relevant to downstream requirements, but generally summarize candidate utility through similarity, influence, or distribution matching. 
Beyond candidate scoring, several methods consider how supervision should be distributed across skills or tasks. Skill-It adapts training mixtures according to skill dependencies~\citep{chen2023skillitdatadrivenskillsframework}, while MIG and BIDS dynamically prioritize informative or underrepresented targets~\citep{chen2025mig,dai-etal-2025-improving}. 
Our work complements these approaches by connecting instruction allocation to unequal pretrained support through prior barrier. It combines prior barrier with distinguishing evidence to adaptively allocate the SFT budget.

\section{Conclusion}
\label{sec:conclusion}

This work studies supervised fine-tuning under long-tail distributions, where different target concepts can receive substantially different levels of pretrained support. We introduce the \emph{prior barrier} to characterize this difference, and observe that prior barriers exhibit a long-tail distribution. Our theoretical analysis shows that the prior barrier and instruction evidence jointly determine the number of instructions required during SFT: head concepts require relatively limited additional supervision, whereas tail concepts require more evidence to overcome their prior barriers.
Guided by this analysis, we propose \textsc{PASS}, which estimates concept-specific distinguishing evidence and adaptively allocates a limited SFT budget according to the prior barrier. \textsc{PASS} achieves the best overall performance compared with seven instruction selection baselines in all four backbone-budget settings, while its adaptive allocation consistently improves over uniform allocation. These results show that effective SFT instruction selection should account for both the evidence provided by candidate instructions and the adaptive allocation based on the prior barrier.

\newpage
\bibliographystyle{unsrtnat}
\bibliography{reference}

@inproceedings{ouyang2022training,
  title={Training language models to follow instructions with human feedback},
  author={Ouyang, Long and Wu, Jeffrey and Jiang, Xu and Almeida, Diogo and Wainwright, Carroll and Mishkin, Pamela and Zhang, Chong and Agarwal, Sandhini and Slama, Katarina and Ray, Alex and others},
  booktitle={Advances in Neural Information Processing Systems},
  volume={35},
  year={2022}
}

@inproceedings{zhou2023lima,
  title={{LIMA}: Less is more for alignment},
  author={Zhou, Chunting and Liu, Pengfei and Xu, Puxin and Iyer, Srini and Sun, Jiao and Mao, Yuning and Ma, Xuezhe and Efrat, Avia and Yu, Ping and Yu, Lili and others},
  booktitle={Advances in Neural Information Processing Systems},
  volume={36},
  year={2023}
}

@misc{chen2024alpagasus,
      title={AlpaGasus: Training A Better Alpaca with Fewer Data}, 
      author={Lichang Chen and Shiyang Li and Jun Yan and Hai Wang and Kalpa Gunaratna and Vikas Yadav and Zheng Tang and Vijay Srinivasan and Tianyi Zhou and Heng Huang and Hongxia Jin},
      year={2024},
      eprint={2307.08701},
      archivePrefix={arXiv},
      primaryClass={cs.CL},
      url={https://arxiv.org/abs/2307.08701}, 
}

@inproceedings{li2024quantity,
  title={From quantity to quality: Boosting {LLM} performance with self-guided data selection for instruction tuning},
  author={Li, Ming and Zhang, Yong and Li, Zhitao and Chen, Jiuhai and Chen, Lichang and Cheng, Ning and Wang, Jianzong and Zhou, Tianyi and Xiao, Jing},
  booktitle={Proceedings of the 2024 Conference of the North American Chapter of the Association for Computational Linguistics},
  year={2024}
}

@inproceedings{xia2024less,
  title={{LESS}: Selecting influential data for targeted instruction tuning},
  author={Xia, Mengzhou and Malladi, Sadhika and Gururangan, Suchin and Arora, Sanjeev and Chen, Danqi},
  booktitle={International Conference on Machine Learning},
  year={2024}
}

@misc{hu2022lora,
      title={LoRA: Low-Rank Adaptation of Large Language Models}, 
      author={Edward J. Hu and Yelong Shen and Phillip Wallis and Zeyuan Allen-Zhu and Yuanzhi Li and Shean Wang and Lu Wang and Weizhu Chen},
      year={2021},
      eprint={2106.09685},
      archivePrefix={arXiv},
      primaryClass={cs.CL},
      url={https://arxiv.org/abs/2106.09685}, 
}

@inproceedings{
    cheng2026tads,
    title={Task-Aware Data Selection via Proxy-Label Enhanced Distribution Matching for {LLM} Finetuning},
    author={Hao Cheng and Rui Zhang and Ling Li and Na Di and Jiaheng Wei and Zhaowei Zhu and Bo Han},
    booktitle={The Fourteenth International Conference on Learning Representations},
    year={2026},
    url={https://openreview.net/forum?id=R40WoYbYab}
}

@misc{hendrycks2021mmlu,
      title={Measuring Massive Multitask Language Understanding}, 
      author={Dan Hendrycks and Collin Burns and Steven Basart and Andy Zou and Mantas Mazeika and Dawn Song and Jacob Steinhardt},
      year={2021},
      eprint={2009.03300},
      archivePrefix={arXiv},
      primaryClass={cs.CY},
      url={https://arxiv.org/abs/2009.03300}, 
}

@inproceedings{lin2022truthfulqa,
  title={{TruthfulQA}: Measuring how models mimic human falsehoods},
  author={Lin, Stephanie and Hilton, Jacob and Evans, Owain},
  booktitle={Proceedings of the 60th Annual Meeting of the Association for Computational Linguistics},
  year={2022}
}

@misc{cobbe2021gsm8k,
      title={Training Verifiers to Solve Math Word Problems}, 
      author={Karl Cobbe and Vineet Kosaraju and Mohammad Bavarian and Mark Chen and Heewoo Jun and Lukasz Kaiser and Matthias Plappert and Jerry Tworek and Jacob Hilton and Reiichiro Nakano and Christopher Hesse and John Schulman},
      year={2021},
      eprint={2110.14168},
      archivePrefix={arXiv},
      primaryClass={cs.LG},
      url={https://arxiv.org/abs/2110.14168}, 
}

@inproceedings{suzgun2023bbh,
  title={Challenging {BIG}-Bench tasks and whether chain-of-thought can solve them},
  author={Suzgun, Mirac and Scales, Nathan and Sch{\"a}rli, Nathanael and Gehrmann, Sebastian and Tay, Yi and Chung, Hyung Won and Chowdhery, Aakanksha and Le, Quoc V and Chi, Ed H and Zhou, Denny and Wei, Jason},
  booktitle={Findings of the Association for Computational Linguistics: ACL 2023},
  year={2023}
}

@article{clark2020tydiqa,
    title = "{T}y{D}i {QA}: A Benchmark for Information-Seeking Question Answering in Typologically Diverse Languages",
    author = "Clark, Jonathan H.  and
      Choi, Eunsol  and
      Collins, Michael  and
      Garrette, Dan  and
      Kwiatkowski, Tom  and
      Nikolaev, Vitaly  and
      Palomaki, Jennimaria",
    editor = "Johnson, Mark  and
      Roark, Brian  and
      Nenkova, Ani",
    journal = "Transactions of the Association for Computational Linguistics",
    volume = "8",
    year = "2020",
    address = "Cambridge, MA",
    publisher = "MIT Press",
    url = "https://aclanthology.org/2020.tacl-1.30/",
    doi = "10.1162/tacl_a_00317",
    pages = "454--470"
}

@misc{jiang2023mistral,
      title={Mistral 7B}, 
      author={Albert Q. Jiang and Alexandre Sablayrolles and Arthur Mensch and Chris Bamford and Devendra Singh Chaplot and Diego de las Casas and Florian Bressand and Gianna Lengyel and Guillaume Lample and Lucile Saulnier and Lélio Renard Lavaud and Marie-Anne Lachaux and Pierre Stock and Teven Le Scao and Thibaut Lavril and Thomas Wang and Timothée Lacroix and William El Sayed},
      year={2023},
      eprint={2310.06825},
      archivePrefix={arXiv},
      primaryClass={cs.CL},
      url={https://arxiv.org/abs/2310.06825}, 
}

@misc{qwen2024qwen25,
  title={{Qwen2.5} Technical Report},
  author={{Qwen} and Yang, An and Yang, Baosong and Zhang, Beichen and Hui, Binyuan and others},
  year={2024}, eprint={2412.15115}, archivePrefix={arXiv},
  url={https://arxiv.org/abs/2412.15115}
}

@inproceedings{li2026instruction,
  title={Instruction data selection via answer divergence},
  author={Li, Bo and Wang, Mingda and Zhang, Shikun and Ye, Wei},
  booktitle={Proceedings of the 64th Annual Meeting of the Association for Computational Linguistics (Volume 1: Long Papers)},
  pages={4684--4702},
  year={2026}
}

@InProceedings{pmlr-v202-longpre23a,
  title = 	 {The Flan Collection: Designing Data and Methods for Effective Instruction Tuning},
  author =       {Longpre, Shayne and Hou, Le and Vu, Tu and Webson, Albert and Chung, Hyung Won and Tay, Yi and Zhou, Denny and Le, Quoc V and Zoph, Barret and Wei, Jason and Roberts, Adam},
  booktitle = 	 {Proceedings of the 40th International Conference on Machine Learning},
  pages = 	 {22631--22648},
  year = 	 {2023},
  editor = 	 {Krause, Andreas and Brunskill, Emma and Cho, Kyunghyun and Engelhardt, Barbara and Sabato, Sivan and Scarlett, Jonathan},
  volume = 	 {202},
  series = 	 {Proceedings of Machine Learning Research},
  month = 	 {23--29 Jul},
  publisher =    {PMLR},
  url = 	 {https://proceedings.mlr.press/v202/longpre23a.html}
}

@inproceedings{NEURIPS2023_949f0f8f,
 author = {K\"{o}pf, Andreas and Kilcher, Yannic and von R\"{u}tte, Dimitri and Anagnostidis, Sotiris and Tam, Zhi Rui and Stevens, Keith and Barhoum, Abdullah and Nguyen, Duc and Stanley, Oliver and Nagyfi, Rich\'{a}rd and ES, Shahul and Suri, Sameer and Glushkov, David and Dantuluri, Arnav and Maguire, Andrew and Schuhmann, Christoph and Nguyen, Huu and Mattick, Alexander},
 booktitle = {Advances in Neural Information Processing Systems},
 doi = {10.52202/075280-2064},
 editor = {A. Oh and T. Naumann and A. Globerson and K. Saenko and M. Hardt and S. Levine},
 pages = {47669--47681},
 publisher = {Curran Associates, Inc.},
 title = {OpenAssistant Conversations - Democratizing Large Language Model Alignment},
 url = {https://proceedings.neurips.cc/paper_files/paper/2023/file/949f0f8f32267d297c2d4e3ee10a2e7e-Paper-Datasets_and_Benchmarks.pdf},
 volume = {36},
 year = {2023}
}

@inproceedings{ICLR2024_82eec786,
 author = {Xu, Can and Sun, Qingfeng and Zheng, Kai and Geng, Xiubo and Zhao, Pu and Feng, Jiazhan and Tao, Chongyang and Lin, Qingwei and Jiang, Daxin},
 booktitle = {International Conference on Learning Representations},
 editor = {B. Kim and Y. Yue and S. Chaudhuri and K. Fragkiadaki and M. Khan and Y. Sun},
 pages = {30745--30766},
 title = {WizardLM: Empowering Large Pre-Trained Language Models to Follow Complex Instructions},
 url = {https://proceedings.iclr.cc/paper_files/paper/2024/file/82eec786fdfbbfa53450c5feb7d1ac92-Paper-Conference.pdf},
 volume = {2024},
 year = {2024}
}

@misc{DatabricksBlog2023DollyV2,
    author    = {Mike Conover and Matt Hayes and Ankit Mathur and Jianwei Xie and Jun Wan and Sam Shah and Ali Ghodsi and Patrick Wendell and Matei Zaharia and Reynold Xin},
    title     = {Free Dolly: Introducing the World's First Truly Open Instruction-Tuned LLM},
    year      = {2023},
    url       = {https://www.databricks.com/blog/2023/04/12/dolly-first-open-commercially-viable-instruction-tuned-llm},
    urldate   = {2023-06-30}
}

@misc{alpaca,
  author = {Rohan Taori and Ishaan Gulrajani and Tianyi Zhang and Yann Dubois and Xuechen Li and Carlos Guestrin and Percy Liang and Tatsunori B. Hashimoto },
  title = {Stanford Alpaca: An Instruction-following LLaMA model},
  year = {2023},
  publisher = {GitHub},
  journal = {GitHub repository},
  howpublished = {\url{https://github.com/tatsu-lab/stanford_alpaca}},
}

@misc{yue2023mammothbuildingmathgeneralist,
      title={MAmmoTH: Building Math Generalist Models through Hybrid Instruction Tuning}, 
      author={Xiang Yue and Xingwei Qu and Ge Zhang and Yao Fu and Wenhao Huang and Huan Sun and Yu Su and Wenhu Chen},
      year={2023},
      eprint={2309.05653},
      archivePrefix={arXiv},
      primaryClass={cs.CL},
      url={https://arxiv.org/abs/2309.05653}, 
}

@misc{liu2024makesgooddataalignment,
      title={What Makes Good Data for Alignment? A Comprehensive Study of Automatic Data Selection in Instruction Tuning}, 
      author={Wei Liu and Weihao Zeng and Keqing He and Yong Jiang and Junxian He},
      year={2024},
      eprint={2312.15685},
      archivePrefix={arXiv},
      primaryClass={cs.CL},
      url={https://arxiv.org/abs/2312.15685}, 
}

@InProceedings{pmlr-v202-kandpal23a,
  title = 	 {Large Language Models Struggle to Learn Long-Tail Knowledge},
  author =       {Kandpal, Nikhil and Deng, Haikang and Roberts, Adam and Wallace, Eric and Raffel, Colin},
  booktitle = 	 {Proceedings of the 40th International Conference on Machine Learning},
  pages = 	 {15696--15707},
  year = 	 {2023},
  editor = 	 {Krause, Andreas and Brunskill, Emma and Cho, Kyunghyun and Engelhardt, Barbara and Sabato, Sivan and Scarlett, Jonathan},
  volume = 	 {202},
  series = 	 {Proceedings of Machine Learning Research},
  month = 	 {23--29 Jul},
  publisher =    {PMLR},
  url = 	 {https://proceedings.mlr.press/v202/kandpal23a.html}
}

@inproceedings{razeghi-etal-2022-impact,
    title = "Impact of Pretraining Term Frequencies on Few-Shot Numerical Reasoning",
    author = "Razeghi, Yasaman  and
      Logan IV, Robert L  and
      Gardner, Matt  and
      Singh, Sameer",
    editor = "Goldberg, Yoav  and
      Kozareva, Zornitsa  and
      Zhang, Yue",
    booktitle = "Findings of the Association for Computational Linguistics: EMNLP 2022",
    month = dec,
    year = "2022",
    address = "Abu Dhabi, United Arab Emirates",
    publisher = "Association for Computational Linguistics",
    url = "https://aclanthology.org/2022.findings-emnlp.59/",
    doi = "10.18653/v1/2022.findings-emnlp.59",
    pages = "840--854"
}

@inproceedings{gekhman-etal-2024-fine,
    title = "Does Fine-Tuning {LLM}s on New Knowledge Encourage Hallucinations?",
    author = "Gekhman, Zorik  and
      Yona, Gal  and
      Aharoni, Roee  and
      Eyal, Matan  and
      Feder, Amir  and
      Reichart, Roi  and
      Herzig, Jonathan",
    editor = "Al-Onaizan, Yaser  and
      Bansal, Mohit  and
      Chen, Yun-Nung",
    booktitle = "Proceedings of the 2024 Conference on Empirical Methods in Natural Language Processing",
    month = nov,
    year = "2024",
    address = "Miami, Florida, USA",
    publisher = "Association for Computational Linguistics",
    url = "https://aclanthology.org/2024.emnlp-main.444/",
    doi = "10.18653/v1/2024.emnlp-main.444",
    pages = "7765--7784"
}

@misc{chen2023skillitdatadrivenskillsframework,
      title={Skill-it! A Data-Driven Skills Framework for Understanding and Training Language Models}, 
      author={Mayee F. Chen and Nicholas Roberts and Kush Bhatia and Jue Wang and Ce Zhang and Frederic Sala and Christopher Ré},
      year={2023},
      eprint={2307.14430},
      archivePrefix={arXiv},
      primaryClass={cs.CL},
      url={https://arxiv.org/abs/2307.14430}, 
}

@misc{dai-etal-2025-improving,
      title={Improving Influence-based Instruction Tuning Data Selection for Balanced Learning of Diverse Capabilities}, 
      author={Qirun Dai and Dylan Zhang and Jiaqi W. Ma and Hao Peng},
      year={2026},
      eprint={2501.12147},
      archivePrefix={arXiv},
      primaryClass={cs.CL},
      url={https://arxiv.org/abs/2501.12147}, 
}

@inproceedings{chen2025mig,
  title     = {{MIG}: Automatic Data Selection for Instruction Tuning by
               Maximizing Information Gain in Semantic Space},
  author    = {Chen, Yicheng and Li, Yining and Hu, Kai and Ma, Zerun and
               Ye, Haochen and Chen, Kai},
  booktitle = {Findings of the Association for Computational Linguistics: ACL 2025},
  pages     = {9902--9915},
  publisher = {Association for Computational Linguistics},
  year      = {2025},
  doi       = {10.18653/v1/2025.findings-acl.515},
  url       = {https://aclanthology.org/2025.findings-acl.515/}
}

@inproceedings{liu2024selectit,
  title     = {{SelectIT}: Selective Instruction Tuning for {LLMs} via
               Uncertainty-Aware Self-Reflection},
  author    = {Liu, Liangxin and Liu, Xuebo and Wong, Derek F. and
               Li, Dongfang and Wang, Ziyi and Hu, Baotian and Zhang, Min},
  booktitle = {Advances in Neural Information Processing Systems},
  volume    = {37},
  year      = {2024},
  url       = {https://proceedings.neurips.cc/paper_files/paper/2024/hash/b130a5691815f550977e331f8bec08ae-Abstract-Conference.html}
}

@inproceedings{zhang2025grape,
  title     = {The Best Instruction-Tuning Data are Those That Fit},
  author    = {Zhang, Dylan and Dai, Qirun and Peng, Hao},
  booktitle = {Advances in Neural Information Processing Systems},
  volume    = {38},
  year      = {2025},
  doi       = {10.52202/085713-4717},
  url       = {https://proceedings.neurips.cc/paper_files/paper/2025/hash/cea04322465ad2f261f08e5b47ba9e7a-Abstract-Conference.html}
}

@misc{ivison2025large,
  title         = {Large-Scale Data Selection for Instruction Tuning},
  author        = {Ivison, Hamish and Zhang, Muru and Brahman, Faeze and
                   Koh, Pang Wei and Dasigi, Pradeep},
  year          = {2025},
  eprint        = {2503.01807},
  archivePrefix = {arXiv},
  primaryClass  = {cs.CL},
  note          = {arXiv preprint arXiv:2503.01807},
  doi           = {10.48550/arXiv.2503.01807},
  url           = {https://arxiv.org/abs/2503.01807}
}

@inproceedings{liu2024tsds,
  title     = {{TSDS}: Data Selection for Task-Specific Model Finetuning},
  author    = {Liu, Zifan and Karbasi, Amin and Rekatsinas, Theodoros},
  booktitle = {Advances in Neural Information Processing Systems},
  volume    = {37},
  year      = {2024},
  url       = {https://openreview.net/forum?id=wjbTHLUSzU}
}

\newpage
\appendix
\section{Details of Theoretical Analysis}

\subsection{Assumptions}
\label{app:assumptions}

We assume that the concept space $\Theta\subset\mathbb{R}^d$ is compact and is equipped with a Borel probability reference measure $\lambda$ with full support. The prior has a strictly positive continuous density $\pi(\theta)$ with respect to $\lambda$. The input and output spaces are standard Borel spaces. There exists a common $\sigma$-finite measure $\nu$ on the output space such that each $P_\theta(\cdot\mid x)$ admits a density, denoted by $P_\theta(y\mid x)$, with $(\theta,x,y)\mapsto P_\theta(y\mid x)$ jointly measurable. Define $P^\star(dx,dy)=P_X(dx)P_{\theta^\star}(dy\mid x)$, and assume that the SFT examples are i.i.d. from $P^\star$. Let $\ell_\theta(x,y)=\log P_\theta(y\mid x)$ and $S_n(\theta)=\sum_{i=1}^n\ell_\theta(x_i,y_i)$.

\begin{assumption}[]
\label{assum:wellposed-likelihood}
There exists a measurable set $\mathcal{Z}_0$ of $P^\star$-probability one such that, for every $(x,y)\in\mathcal{Z}_0$,
\begin{equation}
\ell_\theta(x,y)\in\mathbb{R}
\quad\text{for every }\theta\in\Theta,
\quad
\sup_{\theta\in\Theta}\ell_\theta(x,y)<\infty.
\label{eq:wellposed-likelihood}
\end{equation}
\end{assumption}

\begin{assumption}[Bounded log-likelihood ratios]
\label{assum:bounded-ratio}
For a fixed $r>0$, there exists $M_r\in(0,\infty)$ such that, for every $(x,y)\in\mathcal{Z}_0$ and every $\theta\in A_r$,
\begin{equation}
\left|\ell_\theta(x,y)-\ell_{\theta^\star}(x,y)\right|\le M_r.
\label{eq:bounded-ratio}
\end{equation}
\end{assumption}

\begin{assumption}[]
\label{assum:lipschitz-local}
For a fixed $r>0$, there exists $L_r\in(0,\infty)$ such that, for every $(x,y)\in\mathcal{Z}_0$,
\begin{align}
|\ell_\theta(x,y)-\ell_{\theta'}(x,y)|
&\le L_r\|\theta-\theta'\|,
&&\theta,\theta'\in A_r,
\label{eq:lipschitz-competitors}\\
|\ell_\theta(x,y)-\ell_{\theta^\star}(x,y)|
&\le L_r\|\theta-\theta^\star\|,
&&\|\theta-\theta^\star\|<r/2.
\label{eq:lipschitz-neighborhood}
\end{align}
\end{assumption}

\begin{assumption}[]
\label{assum:alpha}
The local approximation term $\alpha(r)$ is finite, and the pointwise supremum used to define it is measurable.
\end{assumption}

For posterior concentration, define
\begin{equation}
\rho_r
=
\min\left\{\frac{r}{2},\frac{\epsilon(r)}{8L_r}\right\},
\quad
U_r
=
\{\theta\in\Theta:\|\theta-\theta^\star\|<\rho_r\},
\label{eq:local-neighborhood}
\end{equation}
and let
\begin{equation}
m(r)
=
\int_{U_r}
\frac{\pi(\theta)}{\pi(\theta^\star)}\,\lambda(d\theta).
\label{eq:local-prior-mass}
\end{equation}
Because $\rho_r>0$, $\lambda$ has full support, and $\pi$ is strictly positive and continuous, $m(r)>0$.

Let
\begin{equation}
\mathcal{N}_r
=
N\left(A_r,\frac{\epsilon(r)}{8L_r},\|\cdot\|\right)
\label{eq:covering-number}
\end{equation}
be the covering number of $A_r$. Compactness of $A_r$ implies $\mathcal{N}_r<\infty$. For $\eta\in(0,1)$, define
\begin{equation}
n_{\mathrm{sep}}(r,\eta)
=
\frac{8M_r^2}{\epsilon(r)^2}
\log\left(\frac{\mathcal{N}_r}{\eta}\right).
\label{eq:uniform-separation-sample}
\end{equation}
Finally, let
\begin{equation}
\Delta_n(\theta)=S_n(\theta)-S_n(\theta^\star).
\label{eq:empirical-gap}
\end{equation}

\subsection{Proof of Theorem~\ref{thm:sft-contraction}}
\label{app:contraction-proof}

\begin{proof}
We first establish a uniform empirical separation over $A_r$. For any fixed $\theta\in A_r$, define
\[
Z_i(\theta)
=
\ell_\theta(x_i,y_i)-\ell_{\theta^\star}(x_i,y_i).
\]
By Assumption~\ref{assum:bounded-ratio}, $Z_i(\theta)\in[-M_r,M_r]$, while the definition of $\epsilon(r)$ gives
\[
\mathbb{E}Z_i(\theta)
=-K(\theta)
\le
-\epsilon(r).
\]
Hence,
\[
\left\{\Delta_n(\theta)\ge-\frac{n\epsilon(r)}{2}\right\}
\subseteq
\left\{\Delta_n(\theta)-\mathbb{E}\Delta_n(\theta)
\ge\frac{n\epsilon(r)}{2}\right\}.
\]
Hoeffding's inequality yields
\begin{equation}
\Pr\left\{\Delta_n(\theta)\ge-\frac{n\epsilon(r)}{2}\right\}
\le
\exp\left(-\frac{n\epsilon(r)^2}{8M_r^2}\right).
\label{eq:hoeffding-fixed}
\end{equation}

Choose a cover $\mathcal{C}_r\subset A_r$ with $|\mathcal{C}_r|=\mathcal{N}_r$ and radius $\epsilon(r)/(8L_r)$. By a union bound,
\[
\Pr\left\{\exists\theta_c\in\mathcal{C}_r:
\Delta_n(\theta_c)\ge-\frac{n\epsilon(r)}{2}\right\}
\le
\mathcal{N}_r
\exp\left(-\frac{n\epsilon(r)^2}{8M_r^2}\right).
\]
When $n\ge n_{\mathrm{sep}}(r,\eta)$, the right-hand side is at most $\eta$. On the complementary event, every $\theta\in A_r$ has a center $\theta_c\in\mathcal{C}_r$ such that, by Assumption~\ref{assum:lipschitz-local},
\[
|\Delta_n(\theta)-\Delta_n(\theta_c)|
\le
\sum_{i=1}^nL_r\|\theta-\theta_c\|
\le
\frac{n\epsilon(r)}{8}.
\]
Therefore, with probability at least $1-\eta$,
\begin{equation}
\sup_{\theta\in A_r}\Delta_n(\theta)
\le
-\frac{3n\epsilon(r)}{8}
\le
-\frac{n\epsilon(r)}{4}.
\label{eq:uniform-separation-event}
\end{equation}

On this event, the posterior mass of $A_r$ can be written as
\begin{equation}
q_n^\star(A_r)
=
\frac{\displaystyle
\int_{A_r}
\exp\bigl(\Delta_n(\theta)\bigr)
\frac{\pi(\theta)}{\pi(\theta^\star)}\lambda(d\theta)}
{\displaystyle
\int_\Theta
\exp\bigl(\Delta_n(\theta)\bigr)
\frac{\pi(\theta)}{\pi(\theta^\star)}\lambda(d\theta)}.
\label{eq:q-ratio}
\end{equation}
For every $\theta\in A_r$,
\[
\log\frac{\pi(\theta)}{\pi(\theta^\star)}
+
\Delta_n(\theta)
\le
B(r)-\frac{n\epsilon(r)}{4}.
\]
Since $\lambda(A_r)\le1$, the numerator of~\ref{eq:q-ratio} satisfies
\begin{equation}
\int_{A_r}
\exp\bigl(\Delta_n(\theta)\bigr)
\frac{\pi(\theta)}{\pi(\theta^\star)}\lambda(d\theta)
\le
\exp\left(B(r)-\frac{n\epsilon(r)}{4}\right).
\label{eq:posterior-numerator}
\end{equation}

For every $\theta\in U_r$, Assumption~\ref{assum:lipschitz-local} and the definition of $\rho_r$ imply
\[
\Delta_n(\theta)
\ge
-\sum_{i=1}^n
|\ell_\theta(x_i,y_i)-\ell_{\theta^\star}(x_i,y_i)|
\ge
-\frac{n\epsilon(r)}{8}.
\]
Restricting the denominator of~\ref{eq:q-ratio} to $U_r$ gives
\begin{equation}
\int_\Theta
\exp\bigl(\Delta_n(\theta)\bigr)
\frac{\pi(\theta)}{\pi(\theta^\star)}\lambda(d\theta)
\ge
m(r)
\exp\left(-\frac{n\epsilon(r)}{8}\right).
\label{eq:posterior-denominator}
\end{equation}
Combining~\ref{eq:posterior-numerator} and~\ref{eq:posterior-denominator},
\[
q_n^\star(A_r)
\le
\frac{1}{m(r)}
\exp\left(B(r)-\frac{n\epsilon(r)}{8}\right).
\]
Together with $q_n^\star(A_r)\le1$, we obtain
\[
q_n^\star(A_r)
\le
\min\left\{
1,
\frac{1}{m(r)}
\exp\left(B(r)-\frac{n\epsilon(r)}{8}\right)
\right\}.
\]
This completes the proof.
\end{proof}

\subsection{Proof of Corollary~\ref{cor:head-tail-sample}}
\label{app:head-tail-proof}

\begin{proof}
In the head case, $B(r)\le0$, so the posterior concentration bound becomes
\[
q_n^\star(A_r)
\le
\min\left\{
1,
\frac{1}{m(r)}
\exp\left(-\frac{n\epsilon(r)}{8}\right)
\right\}.
\]
To guarantee $q_n^\star(A_r)\le\delta$, it is sufficient that
\[
\frac{1}{m(r)}
\exp\left(-\frac{n\epsilon(r)}{8}\right)
\le\delta,
\]
which gives
\[
n
\ge
\frac{8}{\epsilon(r)}
\left(
\log\frac{1}{m(r)\delta}
\right)_+.
\]

In the tail case, $B(r)>0$, and it is sufficient that
\[
\frac{1}{m(r)}
\exp\left(B(r)-\frac{n\epsilon(r)}{8}\right)
\le\delta.
\]
Rearranging gives
\[
n
\ge
\frac{8}{\epsilon(r)}
\left(
B(r)+\log\frac{1}{m(r)\delta}
\right)_+.
\]
This completes the proof.
\end{proof}

\subsection{Proof of Theorem~\ref{thm:risk-decomposition}}
\label{app:risk-proof}

\begin{proof}
Fix an SFT dataset $\mathcal{D}_n$. For $P_X$-almost every $x$, convexity of KL divergence in its second argument gives
\begin{align}
&\KL\left(
P_{\theta^\star}(\cdot\mid x)
\Vert
P_{q_n^\star}(\cdot\mid x)
\right)\\
&\qquad\le
\int_\Theta
\KL\left(
P_{\theta^\star}(\cdot\mid x)
\Vert
P_\theta(\cdot\mid x)
\right)
q_n^\star(d\theta).
\label{eq:kl-convexity}
\end{align}
Splitting the integral over $\Theta\setminus A_r$ and $A_r$, the contribution from $\Theta\setminus A_r$ is bounded by the local approximation term $\alpha(r)$. For $\theta\in A_r$, Assumption~\ref{assum:bounded-ratio} implies
\[
\KL\left(
P_{\theta^\star}(\cdot\mid x)
\Vert
P_\theta(\cdot\mid x)
\right)
=
\mathbb{E}_{y\sim P_{\theta^\star}(\cdot\mid x)}
\left[
\ell_{\theta^\star}(x,y)-\ell_\theta(x,y)
\right]
\le M_r.
\]
Taking expectation over $x$ therefore gives
\[
\mathcal{E}_n
\le
\alpha(r)
+
M_r q_n^\star(A_r).
\]

Under the conditions of Theorem~\ref{thm:sft-contraction}, with probability at least $1-\eta$,
\[
q_n^\star(A_r)
\le
\min\left\{
1,
\frac{1}{m(r)}
\exp\left(B(r)-\frac{n\epsilon(r)}{8}\right)
\right\}.
\]
Substituting this bound yields
\[
\mathcal{E}_n
\le
\alpha(r)
+
M_r
\min\left\{
1,
\frac{1}{m(r)}
\exp\left(B(r)-\frac{n\epsilon(r)}{8}\right)
\right\}.
\]
This completes the proof.
\end{proof}

\section{Pseudo Code}
\label{app:pseudocode}

\begin{algorithm}[H]
\caption{\textsc{PASS}: Prior-barrier Aware SFT Instruction Selection}
\label{alg:pass}
\small
\begin{algorithmic}[1]
\Require Candidate pool $\mathcal{D}$, reference set $\mathcal{T}$, pretrained model $M_0$, selection budget $n$
\Ensure Selected subset $\mathcal{D}_n$ with $|\mathcal{D}_n|=n$

\Statex \textbf{M1: Distinguishing Evidence Estimation}

\ForAll{$(x,y)\in\mathcal{T}$}
    \State Encode $(x,y)$ with $M_0$ to obtain the response representation $z(x,y)$
\EndFor

\State Hierarchically cluster the reference representations using the coherence criterion
\State Obtain target concepts $\Theta=\{\theta_1,\ldots,\theta_C\}$ and their reference subsets $\{\mathcal{T}_\theta\}_{\theta\in\Theta}$

\ForAll{$\theta\in\Theta$}
    \State Compute the desired representation $p_\theta$ from $\mathcal{T}_\theta$
    \State Collect wrong-response instructions associated with $\theta$
    \State Construct the competing representations $\mathcal{N}_\theta=\{n_{\theta 1},\ldots,n_{\theta K_\theta}\}$
\EndFor

\ForAll{$(x_i,y_i)\in\mathcal{D}$}
    \State Encode $(x_i,y_i)$ to obtain $z_i=z(x_i,y_i)$
    \ForAll{$\theta\in\Theta$}
        \State Compute $s_\theta^{+}(z_i)$ and $s_\theta^{-}(z_i)$ using Eq.~\ref{eq:pass_similarity}
        \State Compute distinguishing evidence $\epsilon_{i\theta}$ using Eq.~\ref{eq:pass_evidence}
    \EndFor
\EndFor

\State Form the evidence matrix $[\epsilon_{i\theta}]\in\mathbb{R}^{N\times C}$

\Statex \textbf{M2: Barrier-Aware Instruction Allocation}

\ForAll{$\theta\in\Theta$}
    \ForAll{$(x,y)\in\mathcal{T}_\theta$}
        \State Generate $\hat y\sim M_0(\cdot\mid x)$ and obtain $\hat z=z(x,\hat y)$
    \EndFor
    \State Compute the prior barrier $B_\theta$ using Eq.~\ref{eq:pass_barrier}
    \State Compute the existing target-concept support $\gamma_\theta$ using Eq.~\ref{eq:pass_support}
\EndFor

\State Compute the allocation weights $\{A_\theta\}_{\theta\in\Theta}$ using Eq.~\ref{eq:pass_allocation_weight}

\State Initialize $\mathcal{S}\gets\varnothing$

\While{$|\mathcal{S}|<n$}
    \ForAll{$\theta\in\Theta$}
        \State Compute $R_\theta^{+}(\mathcal{S})$ and $R_\theta^{-}(\mathcal{S})$
    \EndFor
    \State Compute the redundancy term $D(\mathcal{S})$
    \State Evaluate the selection objective $F(\mathcal{S})$ using Eq.~\ref{eq:pass_objective}
    \State Select $i^{\star}$ by the marginal-gain rule in Eq.~\ref{eq:pass_greedy_selection}
    \State $\mathcal{S}\gets\mathcal{S}\cup\{(x_{i^{\star}},y_{i^{\star}})\}$
\EndWhile

\State \Return $\mathcal{D}_n\gets\mathcal{S}$

\end{algorithmic}
\end{algorithm}

\section{Details of Experiments}

\subsection{Implementation Details}
\label{app:detail-setup}
We evaluate the five capabilities in the downstream task using \textsc{MMLU}~\citep{hendrycks2021mmlu}, \textsc{BBH}~\citep{suzgun2023bbh}, \textsc{GSM8K}~\citep{cobbe2021gsm8k}, \textsc{TruthfulQA}~\citep{lin2022truthfulqa}, and \textsc{TyDiQA}~\citep{clark2020tydiqa}, respectively. Following prior work, each capability benchmark is partitioned into $20\%$ reference data and $80\%$ evaluation data. The reference partitions from all five benchmarks are combined into a single reference set $\mathcal{T}$ representing the desired behaviors of the downstream task, while the remaining examples are used for evaluation. The resulting evaluation sets contain $11{,}250$, $1{,}080$, $200$, $632$, and $885$ examples for \textsc{MMLU}, \textsc{BBH}, \textsc{GSM8K}, \textsc{TruthfulQA}, and \textsc{TyDiQA}, respectively. All experiments are conducted on NVIDIA A100 GPUs with 80GB memory.

\subsection{Runtime and Memory Consumption}
\label{sec:complexity}
Figure~\ref{fig:pass-complexity} summarizes the total runtime and peak memory usage during instruction selection. As shown, \textsc{LESS}, \textsc{TSDS}, and \textsc{ADG} incur substantially higher running time, while our method requires considerably less time. Our method has the lowest CPU memory usage among the non-random selectors and comparable GPU memory usage.

\begin{figure}[t]
    \centering
    \includegraphics[width=0.9\linewidth]{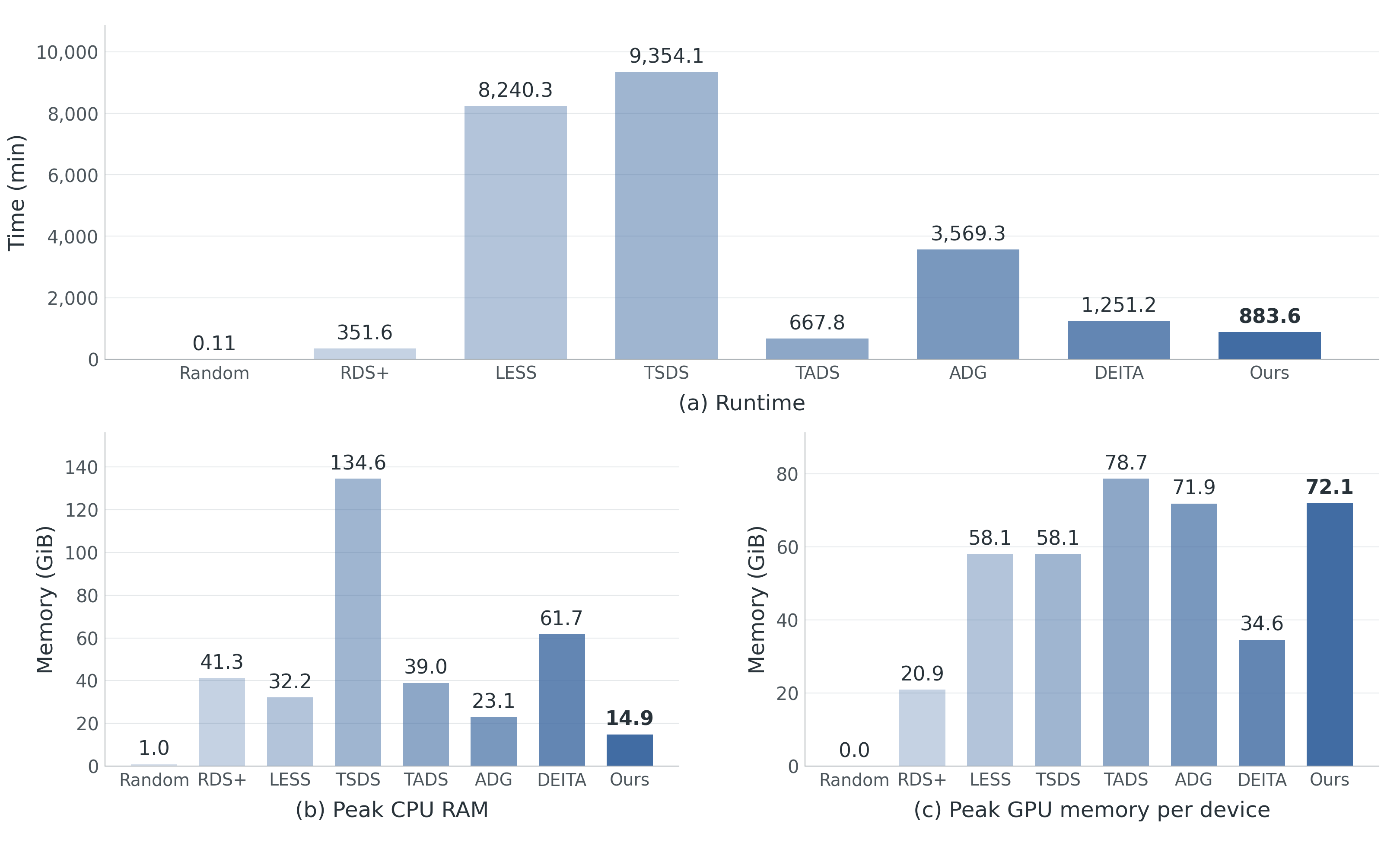}
    \caption{Runtime and memory consumption of all selection methods. (a) Total runtime in minutes. (b) Peak CPU memory. (c) Peak GPU memory per device.}
    \label{fig:pass-complexity}
\end{figure}

\subsection{Selection Composition}
\label{app:composition}
Figures~\ref{fig:mistral-instruction-distribution} and~\ref{fig:qwen-instruction-distribution} compare the source distributions of the instructions selected under the $5{,}000$ budget. Random selection largely preserves the original pool distribution, whereas \textsc{PASS} substantially reweights the source mixture. On both Mistral-7B and Qwen2.5-7B, \textsc{PASS} increases the proportions of FLAN V2, OASST1, and Dolly while reducing the proportion of WizardLM relative to the candidate pool. The selected distributions also vary across the two backbones, reflecting the model-dependent nature of the prior barrier and instruction evidence used by \textsc{PASS}. In contrast, several baselines exhibit much stronger concentration on a small number of sources, such as TADS and TSDS. These results suggest that \textsc{PASS} adapts the composition of the selected subset rather than simply preserving the source distribution or concentrating on a single source.

\begin{figure}[t]
    \centering
    \begin{minipage}[t]{\linewidth}
        \centering
        \includegraphics[width=0.9\linewidth]{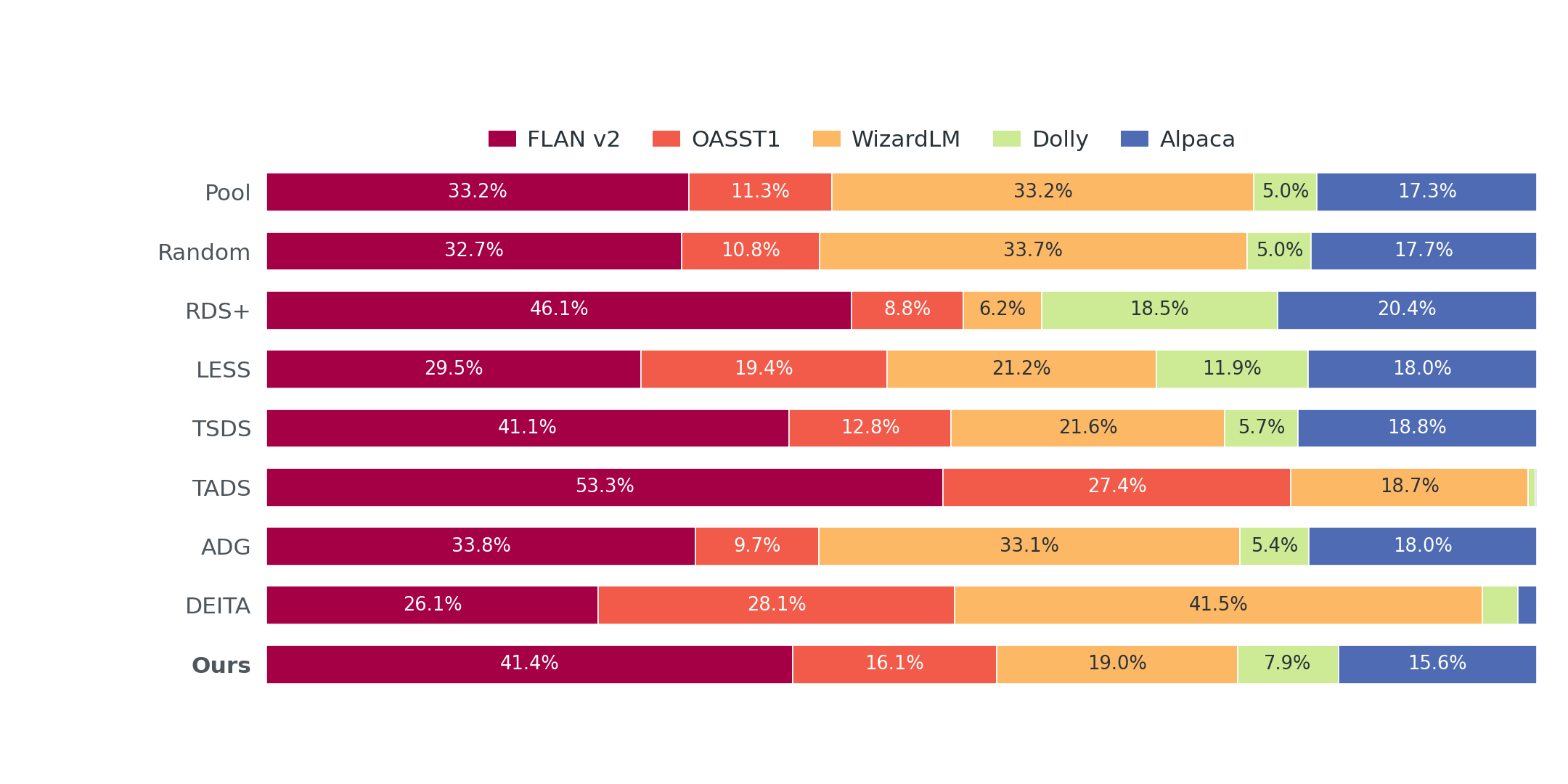}
        \caption{Distribution of instructions selected by Mistral-7B under the 5K selection budget.}
        \label{fig:mistral-instruction-distribution}
    \end{minipage}
    \begin{minipage}[t]{\linewidth}
        \centering
        \includegraphics[width=0.9\linewidth]{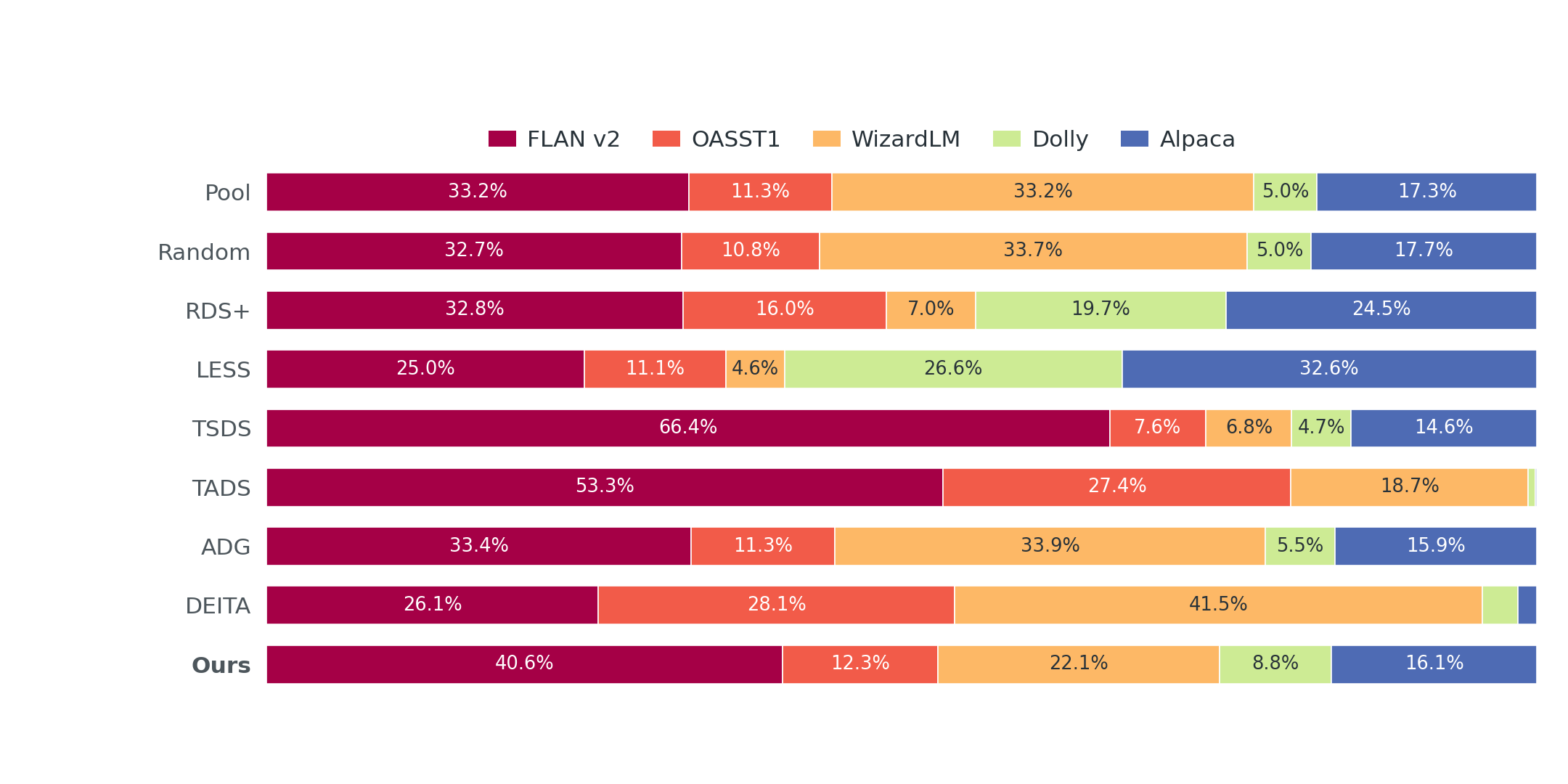}
        \caption{Distribution of instructions selected by Qwen2.5-7B under the 5K selection budget.}
        \label{fig:qwen-instruction-distribution}
    \end{minipage}
\end{figure}

\end{document}